\documentclass[11pt]{article}

\usepackage[margin=0.82in]{geometry}
\usepackage{amsmath,amssymb,amsthm}
\usepackage{booktabs}
\usepackage{graphicx}
\usepackage{microtype}
\usepackage{xcolor}
\usepackage[hidelinks]{hyperref}
\usepackage{url}
\usepackage{array}
\usepackage{enumitem}
\usepackage{caption}
\usepackage{placeins}
\newtheorem{theorem}{Theorem}
\newtheorem{proposition}{Proposition}
\newcommand{\E}{\mathbb{E}}
\newcommand{\R}{\mathbb{R}}
\newcommand{\tr}{\operatorname{tr}}

\newcommand{\NMSE}{\operatorname{NMSE}}
\newcommand{\method}{MS-SRD}
\hypersetup{
    pdftitle={Training-Free Bottleneck Width Planning for Convolutional Autoencoders},
    pdfauthor={Guannan Guo},
    pdfsubject={Training-free multiscale spectral bottleneck selection}
}

\title{Training-Free Bottleneck Width Planning\\
for Convolutional Autoencoders}
\author{Guannan Guo\\Beihang University}
\date{}

\begin{document}
\maketitle

\begin{abstract}
Multiscale Spectral Rate--Distortion (\method) estimates the bottleneck channels required at user-supplied spatial cuts from training images and a normalized mean-squared error (NMSE) bound, without fitting a neural network. Its covariance-tail rule is exact for shared linear block-convolutional autoencoders under squared error. A nested-scale dominance result motivates reporting the activation--parameter Pareto frontier alongside the minimum-latent candidate. At NMSE $\leq0.01$ on thirteen grayscale datasets, its latent-size prediction has 0.84\% mean absolute percentage error against nonlinear patch-autoencoder boundaries; ten predictions are exact and the remaining three differ by one channel. In a four-dataset deployable comparison, \method\ matches all retrospective external widths and all four selected models pass, without training a selector; a 46-fit validation grid and four Least-Volume fits each pass on two datasets. In a skip-closed U-shaped autoencoder at the same bound, five predictions are exact, nine are within one channel, and every failing prediction is one channel short. Experiments at looser bounds show progressively larger nonlinear savings.
\end{abstract}

\section{Introduction}

Let $X\in\R^{H\times W\times C}$ denote an input image of height $H$, width $W$, and $C$ channels. An encoder--decoder maps $X$ to a latent tensor $Z\in\R^{H_z\times W_z\times C_z}$, whose spatial height, spatial width, and channel count are $H_z$, $W_z$, and $C_z$, respectively, and reconstructs $X$ from $Z$. Choosing these three latent dimensions is consequential: an undersized bottleneck loses reconstructive fidelity, while an oversized bottleneck wastes computation and weakens the intended compression.

Current practice broadly follows three patterns. First, an architecture template may supply an empirical width schedule; the original U-Net, for example, doubles the number of feature channels at each downsampling step \cite{ronneberger2015u}. Second, when bottleneck width is treated as a hyperparameter, repeated training, grid search, or an ablation over candidate widths is used to locate an acceptable reconstruction boundary \cite{boquet2021theoretical,bonheme2022fondue}. Third, adaptive approaches can learn an ordering or variable effective width during neural optimization \cite{rippel2014learning,koike2020rateless,ho2025information}. Each strategy can be appropriate, but all three either inherit a width convention or require trained networks before the dataset-specific bottleneck is known.

The training data already contain strong evidence about the required capacity. Pixel correlations, local smoothness, edges, and multiscale structure determine how quickly the covariance spectrum decays. Principal component analysis (PCA) turns that spectrum into an exact squared-error dimension for an unrestricted linear map. A convolutional bottleneck, however, is not an unrestricted vector code: its spatial grid repeats a local representation at many positions. Global PCA therefore answers a different question and can severely undercount the scalar width of a spatial bottleneck.

This paper asks a complementary question: whether the images themselves and an NMSE bound suffice to predict a convolutional bottleneck tensor before any network is trained. We answer it with a data-dependent geometry that turns a broad width search into an explicit spectral calculation.

\method, summarized in Fig.~\ref{fig:method}, pools non-overlapping image patches at candidate scales and estimates the smallest channel count satisfying the prescribed NMSE bound at each scale. We write $\mathcal Q$ for the finite set of candidate square block side lengths in pixels, $q\in\mathcal Q$ for one candidate length, $\delta$ for the specified NMSE upper bound, and $c_q$ for the channel count selected at scale $q$. The method reports both the minimum-latent candidate and the non-dominated trade-off between latent activations and shared linear-operator parameters. The spatial scales are admissible architectural cuts supplied by the user; estimation uses neither labels nor network gradients.

Our contributions are:
\begin{enumerate}[leftmargin=1.4em,itemsep=2pt]
    \item A training-free rule that predicts the required channel count at each admissible spatial cut and exposes the activation--parameter Pareto frontier rather than returning only a flat intrinsic dimension.
    \item An exact population result for shared linear block-convolutional autoencoders under MSE, a conservative extension to non-divisor block sizes, and a dominance result explaining why latent count alone favors nested coarse cuts.
    \item A validation against nonlinear convolution-equivalent autoencoders on handwritten, clothing, natural, remote-sensing, and medical images, together with deployable grid-search and Least-Volume baselines.
    \item A true-bottleneck U-shaped autoencoder study across four NMSE budgets, with a full-skip U-Net as an architectural control.
\end{enumerate}

\begin{figure}[t]
    \centering
    \includegraphics[width=\linewidth]{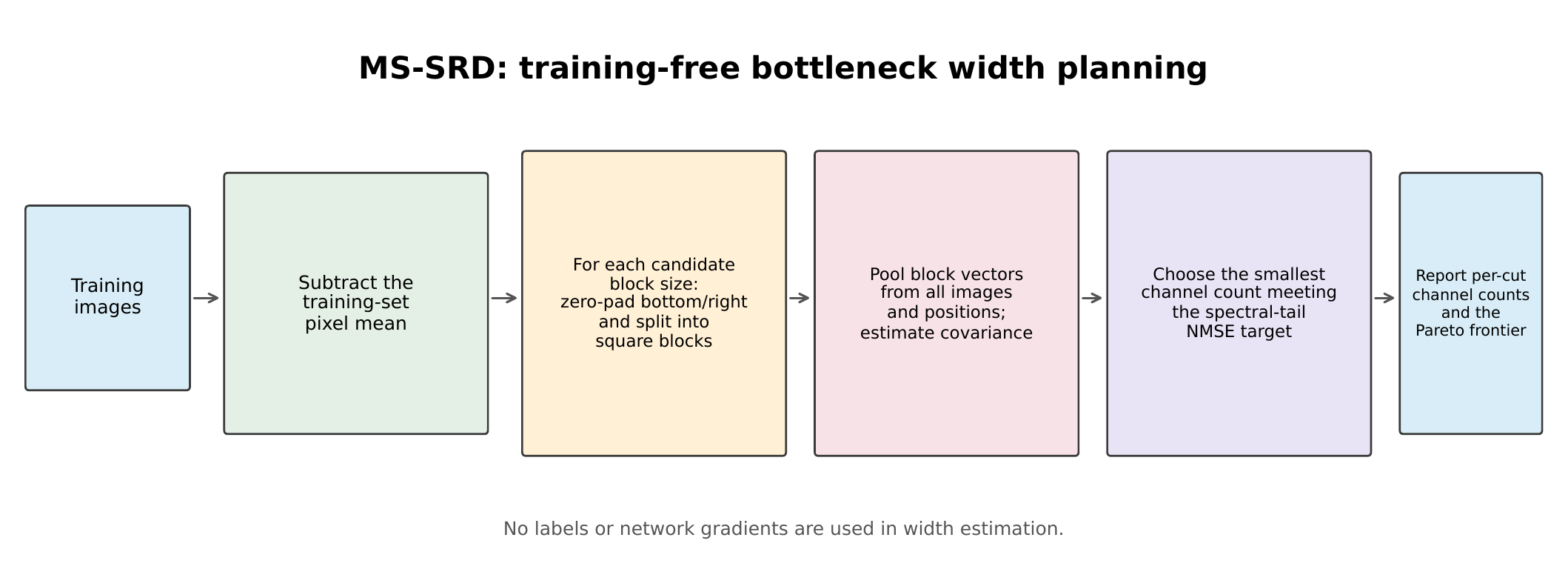}
    \caption{The \method\ pipeline. The mean image is estimated from the training split. ``Pooling patches'' means collecting patch vectors from every training image and spatial block into one covariance sample; it is not a neural pooling layer. Network training is used only for subsequent evaluation.}
    \label{fig:method}
\end{figure}

\section{Related Work}

The starting point for reconstruction under squared error is the classical equivalence between linear autoencoders and principal component analysis (PCA) \cite{baldi1989neural}. The Eckart--Young theorem identifies the best rank-constrained linear approximation \cite{eckart1936approximation}, while Gaussian rate--distortion theory gives a reverse-waterfilling allocation when the representation is quantized and rate is measured in bits \cite{cover2006elements}. These results determine a number of global coordinates, but a convolutional bottleneck has two distinct design variables: how many spatial sites remain and how many channels are stored at each site. A vectorized PCA code does not determine that allocation.

Intrinsic-dimension estimators address a related but different need. Maximum-likelihood nearest-neighbor estimators \cite{levina2004maximum} and the TwoNN estimator \cite{facco2017estimating} infer a scalar manifold dimension from sample distances. Spectral summaries instead soften matrix rank: entropy effective rank exponentiates the entropy of normalized singular values \cite{roy2007effective}, whereas the participation ratio uses the first two moments of the covariance eigenvalues \cite{gao2017theory}. Large-scale studies have applied intrinsic-dimension estimators to image data \cite{pope2021intrinsic}, and data-dependent capacity analyses connect geometry with learnability \cite{yang2022data}. None of these scalar summaries specifies a latent height, width, and channel count under a shared local encoder.

Several autoencoder studies address dimension selection more directly. Tang, Lim, and Siebert train convolutional autoencoders on MNIST and SVHN and compare bottleneck-neuron counts and feature-map counts through reconstruction error, hidden activation statistics, and visualization \cite{tang2018parameterization}. Their recommended parameterization is tied to the evaluated small-image architectures and is obtained only after training the alternatives; it is not conditioned on a prespecified reconstruction bound. Bahadur and Paffenroth add an $L_{2,1}$ penalty to a trained autoencoder and aggregate normalized latent activations into singular-value proxies for intrinsic-dimension estimation \cite{bahadur2020dimension}. This produces a scalar dimension estimate whose value depends on the trained representation and regularization, without allocating that dimension across a convolutional tensor. Least Volume instead penalizes the geometric mean of latent-coordinate standard deviations while constraining the decoder's Lipschitz constant; nearly constant coordinates are pruned after training an initially overcomplete autoencoder \cite{chen2024compressing}. It therefore still requires an initial architecture, neural optimization, a regularization strength, and a numerical pruning rule, and it does not select a spatial bottleneck grid before training.

Zero-cost proxies were developed to reduce the expense of neural architecture search by scoring an initialized candidate from one minibatch or from activation, gradient, or Jacobian statistics, without training the candidate to convergence \cite{abdelfattah2021zero,mellor2021neural}. They primarily provide a relative ranking within a predefined architecture search space rather than an absolute reconstruction-feasible bottleneck width. A candidate network, its parameterization, and usually a sampled minibatch must still be instantiated before a score can be computed. A broad evaluation of 13 proxies on 28 tasks found proxy-specific biases, including sensitivity to cell size, cases in which parameter count or FLOPs were competitive baselines, and limited generalization across benchmarks \cite{krishnakumar2022nas}. These properties make a ranking dependent on the chosen search space and scoring protocol, and the score alone does not specify a latent tensor under an NMSE constraint.

Image models supply a reason to preserve that spatial structure. Classical texture and scattering models exploit local correlations across scale \cite{portilla2000parametric,bruna2013invariant}; convolutional encoders likewise reuse local operators over image positions. This suggests estimating local second-order structure at several block sizes rather than treating each image as an unstructured vector.

Other approaches adapt the bottleneck during or after network training. Nested dropout learns ordered coordinates \cite{rippel2014learning}; rateless and information-ordered autoencoders learn representations usable at multiple widths \cite{koike2020rateless,ho2025information}; and rate--distortion or mutual-information criteria can tune a trained autoencoder \cite{sanchez2014rate,boquet2021theoretical}. FONDUE uses a separately trained diagnostic model \cite{bonheme2022fondue}. These methods reduce manual search in different ways, but they still require neural optimization before the final width is known. The unresolved question is whether the training images alone can provide a defensible channel requirement at a specified convolutional cut before that optimization begins.

\section{Problem Formulation}

For the analysis below, $X$ is a random grayscale image, so $C=1$, and $\E$ denotes expectation over the image distribution. Its population mean image is $\mu=\E[X]$. Given $n$ training images $X_1,\ldots,X_n$, we estimate the sample mean image $\widehat\mu$ at every row $u\in\{1,\ldots,H\}$ and column $v\in\{1,\ldots,W\}$ by
\begin{equation}
    \widehat\mu(u,v)=\frac{1}{n}\sum_{i=1}^{n}X_i(u,v)
\end{equation}
At the population level, let $Y=X-\mu$ denote the centered random image. For the finite training sample, we form $Y_i=X_i-\widehat\mu$. The same sample mean image $\widehat\mu$ is subtracted from validation and test images. Throughout, ``centered'' means this positionwise subtraction; it does not include variance rescaling.

Let $\widehat X$ denote a reconstruction of $X$, and let $\|\cdot\|_F$ denote the Frobenius norm. We evaluate the reconstruction by
\begin{equation}
    \NMSE(\widehat X)
    = \frac{\E\|X-\widehat X\|_F^2}
    {\E\|X-\mu\|_F^2}.
    \label{eq:nmse}
\end{equation}
Let $\delta\in(0,1)$ be a user-specified distortion budget. A reconstruction is admissible when $\NMSE(\widehat X)\leq\delta$.

Let $\mathcal Q$ be a finite set of candidate block side lengths, measured in pixels. Each $q\in\mathcal Q$ denotes one square $q\times q$ block scale and therefore one candidate spatial cut; every candidate satisfies $1\leq q\leq\min(H,W)$, but $q$ need not divide $H$ or $W$. At cut $q$, let $H_{z,q}$ and $W_{z,q}$ denote the latent grid height and width, and let $C_{z,q}$ denote its channel count. We seek the smallest $C_{z,q}$ satisfying the distortion target. The corresponding number of latent activation scalars is
\begin{equation}
    M_q = H_{z,q}W_{z,q}C_{z,q}.
\end{equation}
Let $C_{\rm in}$ denote the number of input channels; $C_{\rm in}=1$ in the grayscale experiments. Let $L_q$ denote the parameter count of a bias-free linear encoder and decoder shared over all $q\times q$ blocks. Then
\begin{equation}
    L_q=2q^2C_{\rm in}C_{z,q}
\end{equation}
because the encoder and decoder each contain $q^2C_{\rm in}C_{z,q}$ weights. We report the non-dominated candidates in $(M_q,L_q)$. If activation storage is the only cost, the minimum-$M_q$ member is a useful summary; an architecture-imposed cut or a parameter budget can select another member of the frontier.

\section{Multiscale Spectral Rate--Distortion}

For each candidate $q\in\mathcal Q$, define the padded image height $\widetilde H_q$, padded image width $\widetilde W_q$, and number of spatial blocks $P_q$ by
\begin{equation}
    \widetilde H_q=q\left\lceil\frac{H}{q}\right\rceil,
    \qquad
    \widetilde W_q=q\left\lceil\frac{W}{q}\right\rceil,
    \qquad
    P_q=H_{z,q}W_{z,q}
        =\left\lceil\frac{H}{q}\right\rceil
        \left\lceil\frac{W}{q}\right\rceil.
\end{equation}
Thus $H_{z,q}=\lceil H/q\rceil$ and $W_{z,q}=\lceil W/q\rceil$. After centering, each image is padded with zeros on its bottom and right edges to size $\widetilde H_q\times\widetilde W_q$. We denote the resulting padded centered random image by $\widetilde Y_q$ and partition it into $P_q$ non-overlapping $q\times q$ blocks.

For training-image index $i\in\{1,\ldots,n\}$ and block-position index $r\in\{1,\ldots,P_q\}$, let $x_{i,r,q}\in\R^{q^2}$ be the vectorized grayscale patch at scale $q$. In this paper, \emph{pooling patches} means treating the $nP_q$ vectors from all images and all block positions as samples for one shared covariance; it does not mean max pooling, average pooling, or a learned aggregation layer. Their empirical patch mean $\overline x_q$ and empirical covariance $\widehat\Sigma_q$ are
\begin{align}
    \overline x_q
        &=\frac{1}{nP_q}\sum_{i=1}^{n}\sum_{r=1}^{P_q}x_{i,r,q},\\
    \widehat\Sigma_q
        &=\frac{1}{nP_q}\sum_{i=1}^{n}\sum_{r=1}^{P_q}
          (x_{i,r,q}-\overline x_q)(x_{i,r,q}-\overline x_q)^\top.
    \label{eq:samplecov}
\end{align}
For the population counterpart, let $x_q\in\R^{q^2}$ be the random patch obtained by drawing an image from the population and a block position uniformly from the $P_q$ positions. Because the population image has been centered positionwise, $\E x_q=0$. Let $\Sigma_q$ be the covariance of $x_q$, let $U_q\in\R^{q^2\times q^2}$ be an orthogonal matrix whose columns are its eigenvectors, and let $\lambda_{q,j}$ be the eigenvalue associated with column $j$ of $U_q$. Ordering the eigenvalues so that $\lambda_{q,1}\geq\cdots\geq\lambda_{q,q^2}\geq0$ gives
\begin{equation}
    \Sigma_q = \E[(x_q-\E x_q)(x_q-\E x_q)^\top]
    = U_q\operatorname{diag}(\lambda_{q,1},\ldots,\lambda_{q,q^2})U_q^\top,
\end{equation}

For $c\in\{0,\ldots,q^2\}$, interpret $c$ as the number of retained patch eigenvectors. The minimum channel count satisfying the patch NMSE bound is
\begin{equation}
    c_q(\delta)=\min\left\{c:\frac{\sum_{j=c+1}^{q^2}\lambda_{q,j}}
    {\sum_{j=1}^{q^2}\lambda_{q,j}}\leq\delta\right\}.
    \label{eq:cq}
\end{equation}
We abbreviate $c_q(\delta)$ as $c_q$ when $\delta$ is fixed. The resulting channel count is $C_{z,q}=c_q$, and the candidate latent tensor and scalar count are
\begin{equation}
    Z_q\in\R^{\lceil H/q\rceil\times\lceil W/q\rceil\times c_q},\qquad
    M_q=P_qc_q.
\end{equation}
For compatibility with a single-number output, define the minimum-latent candidate
\begin{equation}
    q^\dagger=\arg\min_{q\in\mathcal Q}M_q,\qquad
    (H_z,W_z,C_z)=
    (\lceil H/q^\dagger\rceil,\lceil W/q^\dagger\rceil,c_{q^\dagger}).
    \label{eq:select}
\end{equation}
Ties are resolved toward smaller $q$. We use this candidate as the compact scalar-count summary and retain the full activation--parameter frontier for architectural comparison. Pixels that are constant across the training set become zero after centering, and their zero-variance directions cannot increase $c_q$.

\subsection{Exact result for a shared linear block autoencoder}

For a fixed $q\in\mathcal Q$ and integer channel count $c\in\{0,\ldots,q^2\}$, consider an encoder applying one shared matrix $A\in\R^{c\times q^2}$ to every non-overlapping patch and a decoder applying one shared matrix $B\in\R^{q^2\times c}$. This is a convolution and transposed convolution with kernel and stride $q$, no overlap, no bias, and no skip connections.

\begin{theorem}[Population MSE optimum]
Let $\widehat{\widetilde Y}_q$ denote the padded centered image reconstructed from $\widetilde Y_q$ by $A$ and $B$. For a fixed block size $q$ and channel count $c$, the minimum expected reconstruction error over all such shared linear encoder--decoder pairs is
\begin{equation}
    \min_{A,B}\E\|\widetilde Y_q-\widehat{\widetilde Y}_q\|_F^2
    = P_q\sum_{j=c+1}^{q^2}\lambda_{q,j}.
    \label{eq:theorem}
\end{equation}
It is attained by projecting each patch onto the first $c$ eigenvectors of $\Sigma_q$. Consequently, Eq.~\eqref{eq:cq} is the smallest channel count meeting the population padded-domain NMSE target at scale $q$, and Eq.~\eqref{eq:select} is the minimum-latent member of the candidate family. Cropping the reconstruction back to $H\times W$ can only remove nonnegative squared error, while zero padding adds nothing to the target norm. Adding the mean image $\mu$ back after cropping therefore gives a reconstruction on the original image domain that satisfies the same NMSE bound.
\end{theorem}

\begin{proof}
Because the blocks do not overlap, image squared error is the sum of patch squared errors. Let $I\in\R^{q^2\times q^2}$ denote the identity matrix and let $\tr$ denote matrix trace. Pooling positions gives
\begin{align}
\E\|\widetilde Y_q-\widehat{\widetilde Y}_q\|_F^2
&=P_q\E\|x_q-BAx_q\|_2^2\\
&=P_q\tr\left[(I-BA)\Sigma_q(I-BA)^\top\right].
\end{align}
The matrix $BA$ has rank at most $c$. Whitening in the eigencoordinates of $\Sigma_q$ reduces the objective to its best rank-$c$ approximation. If $U_{q,c}\in\R^{q^2\times c}$ contains the first $c$ columns of $U_q$, the Eckart--Young theorem gives the tail sum in Eq.~\eqref{eq:theorem}, attained by $BA=U_{q,c}U_{q,c}^\top$ \cite{eckart1936approximation}. Minimality of $c_q$ and then $M_q$ follows directly.
\end{proof}

\begin{proposition}[Nested-scale dominance]
\label{prop:dominance}
Suppose $q_1,q_2\in\mathcal Q$, $q_2=kq_1$ for a positive integer $k$, and both block sizes divide $H$ and $W$. Then
\begin{equation}
    c_{q_2}(\delta)\leq k^2c_{q_1}(\delta)
    \quad\text{and}\quad
    M_{q_2}\leq M_{q_1}.
\end{equation}
\end{proposition}

\begin{proof}
A $q_2\times q_2$ block contains $k^2$ non-overlapping $q_1\times q_1$ blocks. Applying the feasible rank-$c_{q_1}$ encoder and decoder independently to those sub-blocks gives a rank at most $k^2c_{q_1}$ map with the same relative image-domain distortion. The optimal $q_2$ map cannot require more channels. Since $P_{q_2}=P_{q_1}/k^2$, multiplying the channel inequality by $P_{q_2}$ proves the scalar-count inequality.
\end{proof}

The proposition explains why the minimum-latent summary favors coarser nested cuts in this linear family. Coarser patches also require larger shared transforms, so the $(M_q,L_q)$ frontier remains nontrivial. The experiments consequently report channel counts and latent scalars at explicit admissible cuts together with their activation--parameter trade-off.

\subsection{Estimator and computational cost}

Given $n$ training images, covariance estimation requires accumulating $nP_q$ patch outer products and eigendecomposing a $q^2\times q^2$ matrix. The leading costs are $O(nP_qq^4)$ time and $O(q^4)$ covariance storage per scale; no $HW\times HW$ covariance need be formed. Bootstrap resampling of images quantifies geometry stability.

\section{Experimental Design}

All experiments use fixed, recorded random seeds. Dataset subsampling is uniform without replacement, after which selected indices are sorted. Spectral covariance calculations use float64 accumulation and symmetric eigendecomposition; network tensors and optimization use float32. The reported GPU runs use an NVIDIA RTX 5060 Laptop GPU. Standard GPU kernels can retain limited implementation-level nondeterminism. Image-level bootstrap replicates resample only the development pool, and no label enters spectral estimation or autoencoder training.

The installable implementation, locked \texttt{uv} environment, dataset downloaders, command-line interface, and Python API are available in the accompanying public repository \cite{mssrdsoftware}. From a fresh checkout, \texttt{uv sync} installs the complete environment and \texttt{uv run mssrd reproduce-paper} runs the experiment. Exact seeds, restart behavior, custom-dataset instructions, and the artifact layout are documented in the repository README and machine-readable result files. The remainder of this section specifies the data transformations, candidate set, optimization, model selection, and evaluation boundaries needed to interpret the reported results.

\subsection{Datasets and preprocessing}

We use thirteen public datasets. The $28$--$32$ pixel group contains MNIST \cite{lecun1998gradient}, Kuzushiji-MNIST \cite{clanuwat2018deep}, Fashion-MNIST \cite{xiao2017fashion}, CIFAR-10 and CIFAR-100 \cite{krizhevsky2009learning}, and six MedMNIST v2 datasets \cite{yang2023medmnist}. To test whether the conclusion survives more varied contemporary imagery, we add Oxford-IIIT Pet \cite{parkhi2012cats} and EuroSAT \cite{helber2019eurosat}, both evaluated at $64\times64$. Oxford-IIIT Pet includes substantial variation in breed, pose, scale, and illumination; EuroSAT contributes multispectral satellite scenes, from which we use the released RGB rendering. Labels are not used by the spectral analysis or any autoencoder.

Table~\ref{tab:data} distinguishes the development pool from the external comparison pool. The development pool supplies the training-set mean, covariance estimate, and network-optimization images. The patch-autoencoder experiment uses the external pool retrospectively to construct its empirical comparison boundary, including checkpoint comparison, so that boundary is not presented as an unbiased deployment estimate. The U-shaped experiment instead reserves an internal validation subset from the development pool and evaluates the external pool only after checkpoint selection. We cap these pools at 20,000 and 5,000 images, respectively, and use all available images below those limits. For MedMNIST, the released training and validation images form the development pool and the released test set forms the external pool. Oxford-IIIT Pet uses its official trainval/test split. Because EuroSAT has no official split for this task, a fixed seeded permutation creates an 80/20 split before applying the caps.

Oxford-IIIT Pet images are center-cropped and resized to $64\times64$ with Pillow's antialiased Lanczos resampler \cite{pillow}; EuroSAT is already $64\times64$. For an RGB pixel with red, green, and blue intensities $R$, $G$, and $B$, respectively, we compute the BT.601 grayscale intensity \cite{itu2011bt601} as $I_{\rm gray}=0.299R+0.587G+0.114B$. All intensities are then scaled to $[0,1]$. The per-pixel training mean $\widehat\mu(u,v)$ is estimated only from the development pool and subtracted from every split. Per-pixel variance scaling is omitted because it would reweight the raw-pixel MSE objective in Eq.~\eqref{eq:nmse}. Training-set-constant pixels consequently contribute neither covariance nor selected modes.

\begin{table}[t]
\centering
\caption{Dataset pools after resizing and grayscale conversion. ``Development'' supplies the mean image, covariance estimate, and optimization images. ``External'' supplies the retrospective patch-model comparison boundary. For the U-shaped experiment, an internal validation subset is removed from Development and External is evaluated only after checkpoint selection. Counts are the images actually used after deterministic caps, not necessarily the full released split sizes.}
\label{tab:data}
\small
\begin{tabular}{lrrr@{\hspace{1.2em}}lrrr}
\toprule
Dataset & Shape & Development & External & Dataset & Shape & Development & External\\
\midrule
MNIST & $28^2$ & 20,000 & 5,000 & CIFAR-100 & $32^2$ & 20,000 & 5,000\\
KMNIST & $28^2$ & 20,000 & 5,000 & ChestMNIST & $28^2$ & 20,000 & 5,000\\
Fashion-MNIST & $28^2$ & 20,000 & 5,000 & PneumoniaMNIST & $28^2$ & 5,232 & 624\\
CIFAR-10 & $32^2$ & 20,000 & 5,000 & BreastMNIST & $28^2$ & 624 & 156\\
Oxford-IIIT Pet & $64^2$ & 3,680 & 3,669 & OrganAMNIST & $28^2$ & 20,000 & 5,000\\
EuroSAT & $64^2$ & 20,000 & 5,000 & RetinaMNIST & $28^2$ & 1,200 & 400\\
BloodMNIST & $28^2$ & 13,671 & 3,421 & & & & \\
\bottomrule
\end{tabular}
\end{table}

\subsection{Distortion budgets, candidate scales, and baselines}

All primary patch-model comparisons use $\delta=0.01$. The U-shaped experiment uses the same bound for its main comparison and additionally evaluates $\delta\in\{0.02,0.05,0.10\}$ to measure how the relation changes as the permissible reconstruction error widens.

For every image size, $\mathcal Q=\{2,3,4,5,6,7,8\}$. Non-divisor scales use the zero-padding and cropping rule defined above. The upper limit corresponds to the nominal receptive cell of the three stride-two encoder stages used in Sec.~\ref{sec:trueunet}; allowing still larger blocks would move the comparison toward a global image code rather than the local bottleneck geometry under study. Covariances use every extracted development-pool patch.

For every dataset, we also compute PCA on vectorized centered images at the same distortion budget. Let $d$ denote the retained global PCA dimension. This baseline maps an entire $H\times W$ image to $d$ unrestricted coefficients and is MSE-optimal among rank-$d$ linear reconstructions. A convolutional bottleneck instead stores $C_z$ coefficients at each of $H_zW_z$ locations, for $H_zW_zC_z$ scalars, while reusing the same local encoder across locations. Global PCA ignores this locality and weight-sharing constraint, so it can compare scalar counts but cannot decide how those scalars should be divided among the three tensor axes.

\subsection{Resolution and aspect-ratio stress test}

The thirteen-dataset benchmark tests channel prediction at the spatial cuts used by the validation architectures. It does not, by itself, distinguish the scale axis because every minimum-latent candidate reaches the largest allowed block. We therefore perform a separate controlled experiment on the Oxford-IIIT Pet training split. The same images are center-cropped and resized to $48\times64$, $64\times64$, and $64\times96$, and spectra are estimated at $q\in\{4,8,12,16\}$. This manipulation holds the source images and sample count fixed while changing resolution and aspect ratio. For every geometry, we report $M_q$, the shared linear parameter count $L_q$, and Pareto membership. These three renderings are not counted as independent datasets, and the experiment is not used in the thirteen-dataset prediction metric.

\subsection{Nonlinear validation model}

At each tested pair $(q,c)$ of block scale and channel count, let $x\in\R^{q^2}$ denote one vectorized input patch, let $z\in\R^c$ denote its latent code, and let $\widehat x\in\R^{q^2}$ denote the reconstructed patch. We use the Gaussian error linear unit $\operatorname{GELU}$ of Hendrycks and Gimpel \cite{hendrycks2016gaussian}. Let $h=\min(128,\max(32,2q^2))$ be the residual-branch hidden width. The linear maps are
$A_0:\R^{q^2}\to\R^c$,
$A_1:\R^{q^2}\to\R^h$,
$A_2:\R^h\to\R^c$,
$B_0:\R^c\to\R^{q^2}$,
$B_1:\R^c\to\R^h$, and
$B_2:\R^h\to\R^{q^2}$.
The shared patch autoencoder is
\begin{align}
z &= A_0x + A_2\,\operatorname{GELU}(A_1x),\\
\widehat x &= B_0z + B_2\,\operatorname{GELU}(B_1z).
\end{align}
The maps $A_0$ and $B_0$ form the linear encoder and decoder, $(A_1,A_2)$ form the encoder residual branch, and $(B_1,B_2)$ form the decoder residual branch. The maps $A_0$ and $B_0$ are initialized with the leading patch eigenvectors. The output maps $A_2$ and $B_2$ are initialized to zero, so optimization begins at the exact PCA solution. Sharing the network over non-overlapping patches makes it equivalent to a one-layer strided convolutional encoder and transposed-convolutional decoder augmented by pointwise patch MLPs. There are no skip connections and no access to labels.

For each scale, we test channel counts on a coarse grid containing powers of two, the spectral prediction and its two nearest neighbors on each side, then fill every integer between the highest failing coarse point and the first passing point. The empirical oracle is the smallest \emph{tested} latent scalar count attaining held-out NMSE at most 0.01 over all scales. It is an evaluation reference, not a deployable selection procedure, because constructing it uses held-out reconstruction outcomes.

Training uses AdamW \cite{loshchilov2019decoupled} with learning rate $2\times10^{-3}$, weight decay $10^{-6}$, batch size at most 8,192 patches, at most 250,000 sampled training patches per scale, and 160 update steps. We retain the lowest external-pool NMSE observed at 40-step intervals. A fixed random seed defines the main run. Two additional fixed seeds repeat three points for every dataset: the predicted bottleneck, the empirical boundary, and one channel below the boundary. Implementation uses Python 3.12 and PyTorch 2.11 with CUDA \cite{paszke2019pytorch}.

\subsection{Deployable width-selection baselines}

We compare width selectors on MNIST, CIFAR-10, Oxford-IIIT Pet, and EuroSAT, spanning handwriting, natural images, and remote sensing at $28$, $32$, and $64$ pixels. The spatial cut is fixed to the main experiment's minimum-latent cut ($q=7$ for MNIST and $q=8$ otherwise), so this comparison isolates channel selection. Up to 1,000 images are reserved from the development pool for validation; the mean image, covariance, and network weights use only the remaining images. The external pool is untouched until a width and checkpoint have been selected.

Four rows are reported. \method\ selects $c_q$ without neural training and then fits the nonlinear patch autoencoder once at that width for external evaluation. Block PCA uses the same spectral width and the analytic linear encoder--decoder, so it tests whether neural residual branches are needed rather than acting as an independent width rule. Validation grid search trains the coarse-to-fine candidate set from the nonlinear experiment and selects the smallest validation-passing width. Finally, an adaptation of Least Volume \cite{chen2024compressing} starts with all $q^2$ patch coordinates. If $z_j$ denotes latent coordinate $j$ and $\operatorname{std}(z_j)$ its standard deviation over training patches, the adaptation adds
\begin{equation}
    10^{-3}\exp\left(\frac{1}{q^2}\sum_{j=1}^{q^2}\log(\operatorname{std}(z_j)+1)\right)
\end{equation}
to reconstruction loss, and spectrally normalizes every decoder linear map. Following Chen and Fuge, the adaptation retains the geometric-mean standard-deviation penalty and decoder Lipschitz control. The fixed patch architecture, regularization weight, 800-update budget, and validation-based prefix selection belong to our common comparison protocol rather than their architecture-specific experiments. Coordinates are ordered by training-set standard deviation; the smallest validation-passing prefix is retained, with pruned coordinates fixed at their training mean. The table reports selector training runs, selected width, and external NMSE; the retrospective external boundary remains a comparison target, not an input to any deployable selector.

\subsection{True-bottleneck U-shaped autoencoder}
\label{sec:trueunet}

The principal architecture test uses a U-shaped convolutional autoencoder derived from U-Net \cite{ronneberger2015u}. Its encoder has three stride-two stages with widths $(16,32,64,96)$; the decoder mirrors those stages. All lateral encoder--decoder skips are closed. Consequently, every sample-dependent path crosses one terminal tensor at resolution $\lceil H/8\rceil\times\lceil W/8\rceil$, which matches the selected \method\ grid for all eleven datasets in this architecture experiment. Residual blocks follow the residual-learning construction \cite{he2016deep} and use two $3\times3$ convolutions, Group Normalization \cite{wu2018group}, and GELU activations.

At channel count $c$, the terminal tensor is initialized with the first $c$ patch eigenvectors. The resulting block-PCA reconstruction is an exact input--output path through the terminal tensor. A deep encoder residual is added before the cut and a deep decoder residual after it; their terminal projections are initialized to zero. Training therefore starts at the \method\ linear solution but can learn nonlinear and cross-patch corrections without opening a bypass.

We reserve up to 1,000 images from the development pool for checkpoint selection and train on the remainder for at most 800 AdamW steps, with learning rate $10^{-3}$, weight decay $10^{-6}$, batch size 256, and gradient-norm clipping at 5. For each $\delta$, integer channel counts are searched by bisection under the nested-width monotonicity assumption. The smallest validation-passing width is the deployable selection. The external split is evaluated only after checkpoint selection. For retrospective analysis, we also locate the adjacent external-split crossing; this test oracle measures the architecture's realized boundary but is not a deployable estimator.

\subsection{Full-skip architectural control}

The control restores U-Net's lateral skips and compares a zeroed terminal activation with a one-channel terminal tensor. Zeroing occurs after the bias-free bottleneck projection, so the terminal route carries no sample-dependent value. These candidates use 800 AdamW steps with learning rate $2\times10^{-3}$, weight decay $10^{-6}$, and batch size 256.

For a structural audit, let $S$ be the number of non-batch activation scalars transmitted through all skips, and let $M$ be the predicted terminal latent scalar count. Here
\begin{equation}
S=16HW+32\lceil H/2\rceil\lceil W/2\rceil
  +64\lceil H/4\rceil\lceil W/4\rceil.
\end{equation}
The ratio $S/M$ measures raw activation traffic around the terminal tensor. It is deliberately not interpreted as statistical rank, entropy, or compressed bit rate. The zero-terminal condition, rather than this count alone, tests whether those paths suffice for reconstruction.

\subsection{Uncertainty and metrics}

Twenty nonparametric image-level bootstrap replicates \cite{efron1979bootstrap}, each of size $\min(n,5000)$, re-estimate the full spectral selection. This resampling is especially relevant when the spectral tail lies close to $\delta$: Weyl's eigenvalue perturbation inequality \cite{horn2012matrix} bounds each sample-eigenvalue error by the covariance perturbation norm, but even a small perturbation can move the discrete choice $c_q$ by one. We therefore report the complete bootstrap selection counts rather than only an interval for the selected latent count $M_{q^\dagger}$.

Let $M_{\mathrm{pred}}$ denote the \method-predicted latent scalar count and $M_{\mathrm{emp}}$ the empirical boundary from the nonlinear patch model. Prediction error is measured as $|M_{\mathrm{pred}}-M_{\mathrm{emp}}|/M_{\mathrm{emp}}$. We report mean and median absolute percentage error (MAPE), exact matches, multiplicative tolerances, and Pearson correlation between their logarithms.

\section{Results}

\subsection{Predicted versus empirical bottlenecks on compact images}

Table~\ref{tab:main} contains the compact-image results. The minimum-latent summary uses $q=7$ for the $28\times28$ images and $q=8$ for the $32\times32$ images, yielding $4\times4$ spatial sites. This repeated selection of the largest allowed block follows the nested-scale dominance result; the table evaluates the channel and latent-count prediction at those stated cuts.

Across these eleven compact benchmarks, \method\ achieves 1.00\% MAPE and 0.00\% median absolute percentage error. Eight scalar counts are exact; KMNIST, PneumoniaMNIST, and BreastMNIST differ by one channel, and all eleven predictions are within 5\% of the empirical boundary. The largest error is 4.35\%, and the log-scale Pearson correlation is 0.998. Figure~\ref{fig:prediction} includes these results together with the larger-image extension below.

\begin{table}[t]
\centering
\caption{Training-free predictions and empirical nonlinear boundaries at $\delta=0.01$. $d_{\rm PCA}$ is the global PCA dimension at the same NMSE bound; $M_{\rm pred}$ and $M_{\rm emp}$ are the predicted and empirical latent scalar counts. Error is their absolute percentage difference.}
\label{tab:main}
\scriptsize
\setlength{\tabcolsep}{3.6pt}
\begin{tabular}{lrrrrrrrr}
\toprule
Dataset & $d_{\rm PCA}$ & \multicolumn{3}{c}{\method\ prediction} & \multicolumn{3}{c}{Empirical oracle} & Error\\
\cmidrule(lr){3-5}\cmidrule(lr){6-8}
& & Tensor & $M_{\rm pred}$ & & Tensor & $M_{\rm emp}$ & NMSE & \\
\midrule
MNIST & 330 & $4\times4\times38$ & 608 && $4\times4\times38$ & 608 & .0089 & 0.0\%\\
KMNIST & 496 & $4\times4\times37$ & 592 && $4\times4\times38$ & 608 & .0087 & 2.6\%\\
Fashion-MNIST & 452 & $4\times4\times37$ & 592 && $4\times4\times37$ & 592 & .0092 & 0.0\%\\
CIFAR-10 & 431 & $4\times4\times32$ & 512 && $4\times4\times32$ & 512 & .0092 & 0.0\%\\
CIFAR-100 & 408 & $4\times4\times30$ & 480 && $4\times4\times30$ & 480 & .0097 & 0.0\%\\
ChestMNIST & 151 & $4\times4\times17$ & 272 && $4\times4\times17$ & 272 & .0091 & 0.0\%\\
PneumoniaMNIST & 244 & $4\times4\times24$ & 384 && $4\times4\times25$ & 400 & .0092 & 4.0\%\\
BreastMNIST & 191 & $4\times4\times22$ & 352 && $4\times4\times23$ & 368 & .0091 & 4.3\%\\
OrganAMNIST & 679 & $4\times4\times45$ & 720 && $4\times4\times45$ & 720 & .0094 & 0.0\%\\
RetinaMNIST & 53 & $4\times4\times18$ & 288 && $4\times4\times18$ & 288 & .0093 & 0.0\%\\
BloodMNIST & 315 & $4\times4\times28$ & 448 && $4\times4\times28$ & 448 & .0096 & 0.0\%\\
\bottomrule
\end{tabular}
\end{table}

\begin{figure}[t]
    \centering
    \includegraphics[width=0.93\linewidth]{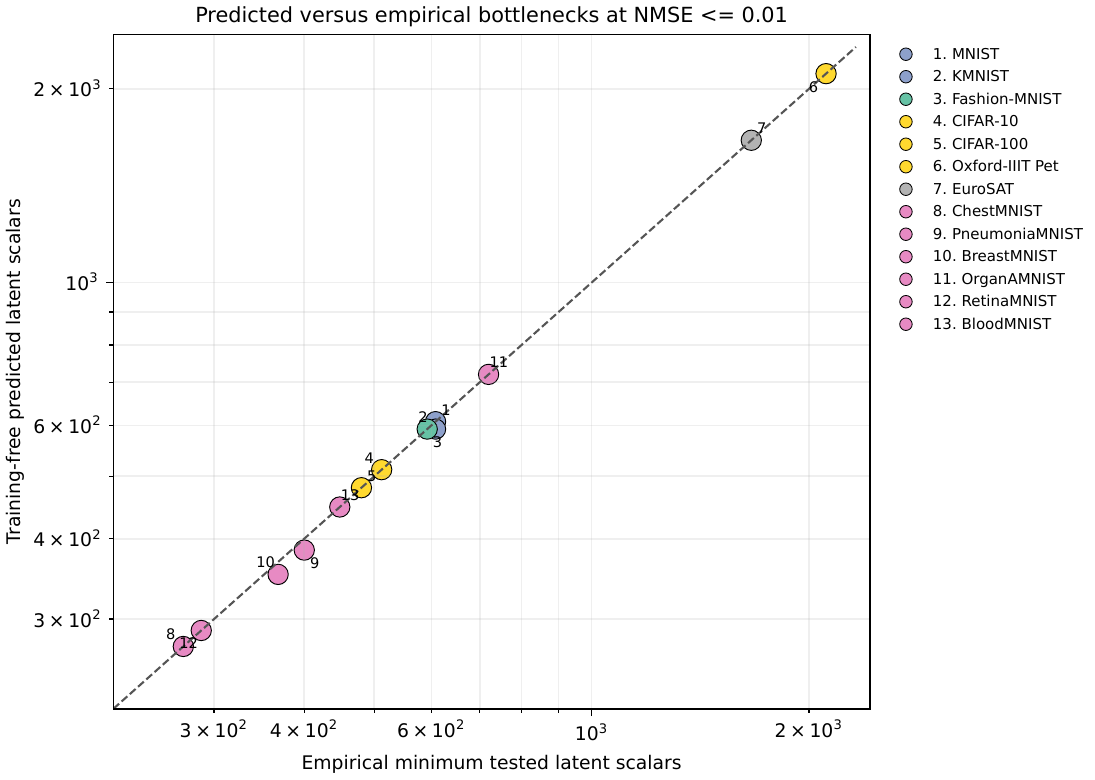}
    \caption{Training-free predicted latent scalar counts versus the empirical minimum tested counts. The dashed line is equality. Numbers identify datasets in the legend.}
    \label{fig:prediction}
\end{figure}

Global PCA has 33.63\% MAPE and underestimates the spatial bottleneck in all thirteen datasets (Fig.~\ref{fig:comparison}). This is expected: a global linear code pays for each retained direction once, whereas a convolutional tensor instantiates $C_z$ channels at every spatial site.

\begin{figure}[t]
    \centering
    \includegraphics[width=\linewidth]{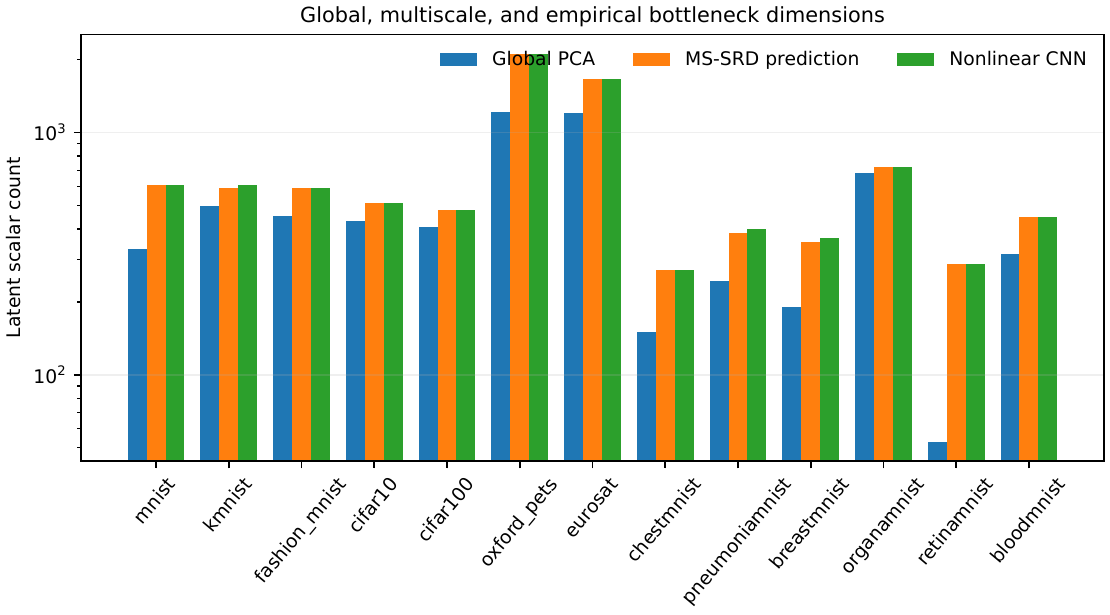}
    \caption{Global PCA, \method, and empirical nonlinear bottleneck sizes. Global PCA uses unrestricted whole-image coefficients; \method\ and the empirical model count coefficients repeated over a spatial bottleneck grid.}
    \label{fig:comparison}
\end{figure}

\subsection{The spatial-cut trade-off}

Figure~\ref{fig:geometry} reports the controlled Oxford-IIIT Pet experiment. Across all three image geometries, $q=4$, $8$, and $16$ are non-dominated, whereas $q=12$ is dominated. Because $q=12$ does not divide at least one image side in any rendering, this status comes from the direct zero-padded estimate, not Proposition~\ref{prop:dominance}. Resolution and aspect ratio change both spatial repetition and the required channel count; the figure shows the resulting exchange between latent activations and shared transform parameters.

These results empirically separate the two cost axes. If only latent activations are charged, $q=16$ wins in every rendering, exactly as the dominance argument predicts. Once shared transform size is included, several scales lie on the Pareto frontier. \method\ records this trade-off as channel requirements, latent activations, and shared transform parameters across the admissible cuts.

\begin{figure}[t]
    \centering
    \includegraphics[width=\linewidth]{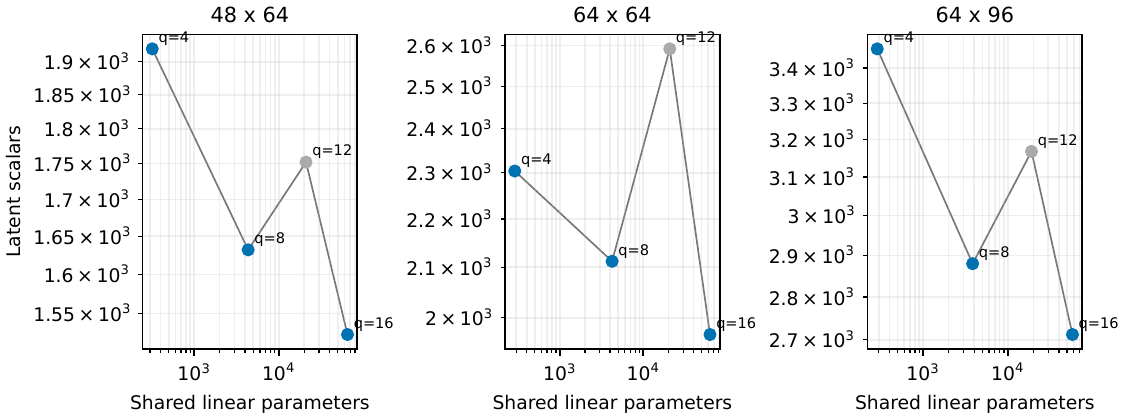}
    \caption{Activation--parameter trade-off for the same Oxford-IIIT Pet training images rendered at three geometries. Blue points are Pareto-optimal; gray points are dominated. Larger blocks reduce repeated latent activations but require larger shared linear transforms.}
    \label{fig:geometry}
\end{figure}

\subsection{Extension to larger, more varied images}

We extend the complete patch-autoencoder protocol to Oxford-IIIT Pet and EuroSAT at $64\times64$. These experiments use the same grayscale conversion, centering rule, distortion bound, candidate scales, optimizer, and channel-search procedure as the compact suite; no dataset-specific calibration is introduced.

Table~\ref{tab:modern} and Fig.~\ref{fig:moderncurves} show the result. \method\ selects $q=8$ for both datasets and exactly matches the empirical boundary: 33 channels for Oxford-IIIT Pet and 26 for EuroSAT. Their full-image PCA dimensions are 1,216 and 1,205, respectively. Bootstrap selections remain confined to adjacent channels at the same scale. Oxford-IIIT Pet selects 33 channels in 12 of 20 resamples and 32 in eight; EuroSAT selects 26 in 14 resamples and 25 in six.

Combining the compact and larger-image experiments gives 0.84\% MAPE across thirteen datasets. Ten predictions are exact, the remaining three differ by one channel, all thirteen are within 5\%, and the Pearson correlation between log latent counts is 0.9996.

\begin{table}[t]
\centering
\caption{Training-free predictions and nonlinear patch-autoencoder boundaries on the $64\times64$ extension at $\delta=0.01$. The bootstrap column reports selection counts over 20 image-level resamples.}
\label{tab:modern}
\small
\setlength{\tabcolsep}{5.2pt}
\begin{tabular}{lrrrrrr}
\toprule
Dataset & Predicted tensor & $M_{\rm pred}$ & Empirical tensor & $M_{\rm emp}$ & NMSE & Bootstrap\\
\midrule
Oxford-IIIT Pet & $8\times8\times33$ & 2,112 & $8\times8\times33$ & 2,112 & .0096 & $c=33$: 12; $c=32$: 8\\
EuroSAT & $8\times8\times26$ & 1,664 & $8\times8\times26$ & 1,664 & .0096 & $c=26$: 14; $c=25$: 6\\
\bottomrule
\end{tabular}
\end{table}

\begin{figure}[t]
    \centering
    \includegraphics[width=\linewidth]{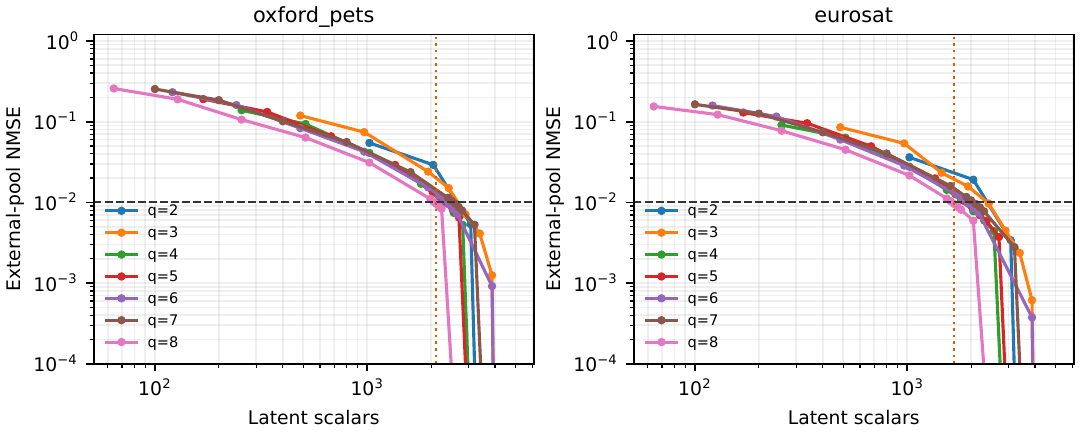}
    \caption{External-pool NMSE for the two $64\times64$ datasets. The dashed horizontal line is $\delta=0.01$ and the dotted vertical line is the training-free predicted scalar count. Both curves cross at the predicted width.}
    \label{fig:moderncurves}
\end{figure}

\subsection{Deployable selector comparison}

Table~\ref{tab:baselines} compares the four selectors under the internal-validation protocol. \method\ matches all four retrospective external widths and all four selected nonlinear models pass. Block PCA shares those widths by construction but narrowly misses the MNIST bound. Validation grid search selects one fewer channel on CIFAR-10 and Oxford-IIIT Pet; both choices pass internally and fail externally. Least Volume makes the same two underselections and adds one channel on MNIST. The stricter bound exposes small validation--external shifts that are hidden at larger reconstruction errors.

In this fixed-cut experiment, the covariance-tail width provides a stronger starting point than the narrowest validation-passing candidate: it matches all four external boundaries while the latter underselects on two datasets.

\begin{table}[t]
\centering
\caption{Deployable fixed-cut width selection on MNIST, CIFAR-10, Oxford-IIIT Pet, and EuroSAT. The channel vector follows that order. ``Fits'' counts neural models used to select width; the one nonlinear fit used only to evaluate the \method-selected width is not charged as selection. MAPE compares channels with the retrospective external oracle.}
\label{tab:baselines}
\small
\setlength{\tabcolsep}{4.5pt}
\begin{tabular}{lrrrrr}
\toprule
Selector & Selected channels & Fits & MAPE & External passes & Mean external NMSE\\
\midrule
\method & 38, 32, 33, 26 & 0 & 0.00\% & 4/4 & .00933\\
Block PCA & 38, 32, 33, 26 & 0 & 0.00\% & 3/4 & .00959\\
Validation grid & 38, 31, 32, 26 & 46 & 1.54\% & 2/4 & .00971\\
Least Volume & 39, 31, 32, 26 & 4 & 2.20\% & 2/4 & .00970\\
\bottomrule
\end{tabular}
\end{table}

\subsection{Spectral and distortion behavior}

Figure~\ref{fig:spectra} shows representative patch spectra. Larger patches expose a longer but rapidly decaying spectrum. The prediction balances this longer channel axis against fewer spatial repetitions. Natural and medical images differ markedly in spectral tail shape, yet the same selection rule applies without dataset-specific fitting.

\begin{figure}[t]
    \centering
    \includegraphics[width=0.96\linewidth]{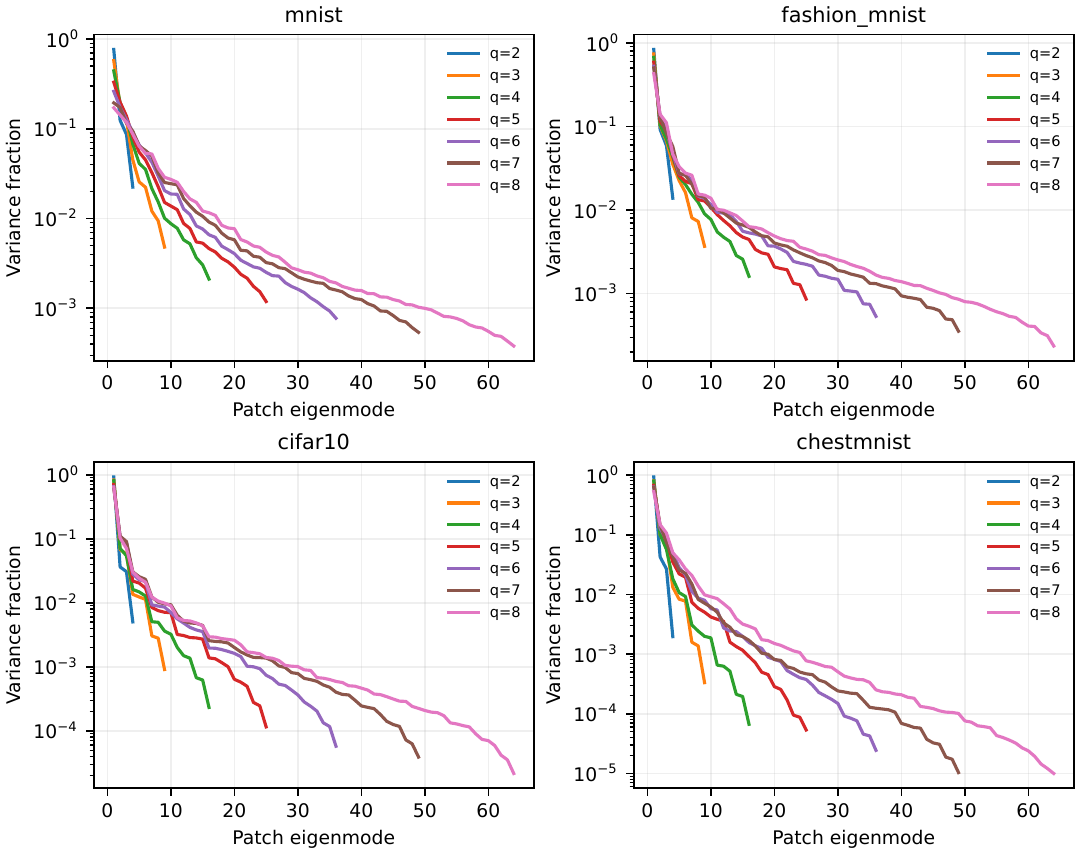}
    \caption{Patch covariance eigenvalues, normalized by total patch variance, at each candidate scale for four representative datasets.}
    \label{fig:spectra}
\end{figure}

The nonlinear distortion curves in Fig.~\ref{fig:curves} are monotone near $\delta=0.01$ and cross the bound close to the vertical spectral prediction. PCA initialization begins at the theorem's linear solution; at this strict bound the residual nonlinear branches leave little room to reduce the required local rank, which is consistent with the close agreement in Table~\ref{tab:main}.

\begin{figure}[t]
    \centering
    \includegraphics[width=0.96\linewidth]{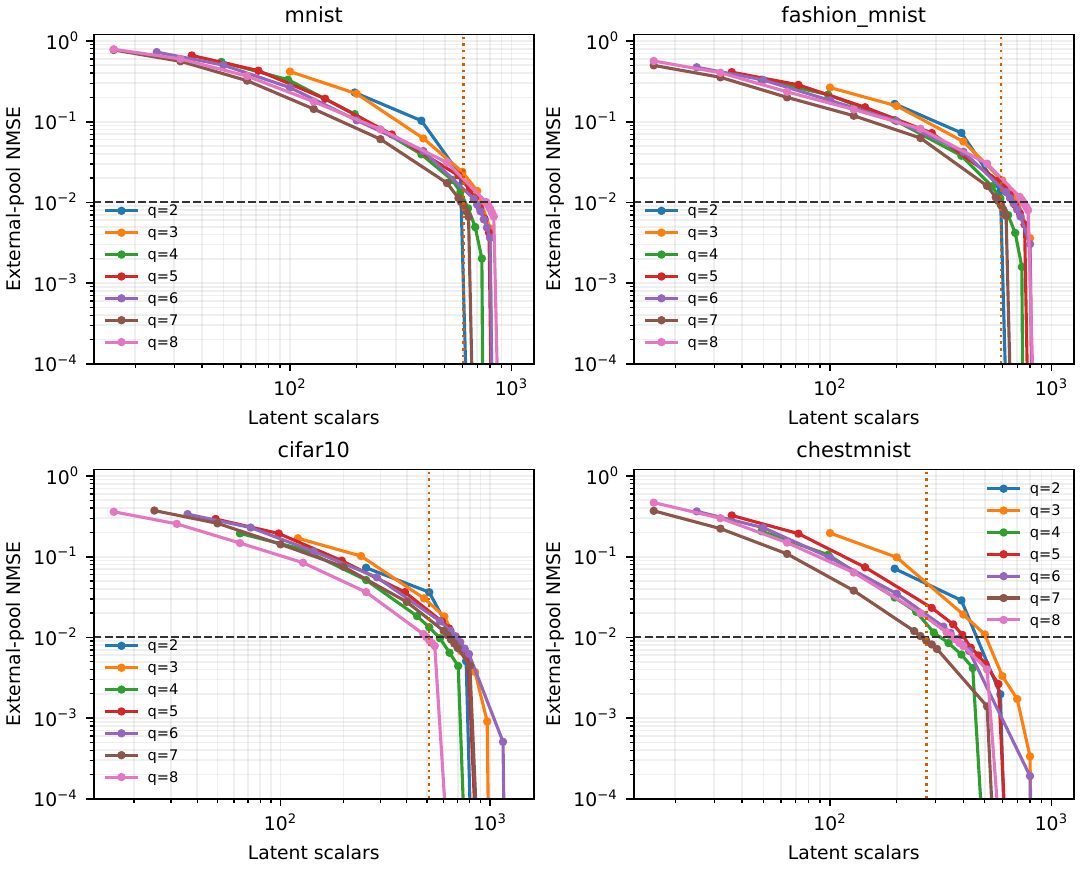}
    \caption{External-pool NMSE against latent scalar count for representative datasets and scales. The horizontal line is $\delta=0.01$; the vertical dotted line is the \method\ prediction.}
    \label{fig:curves}
\end{figure}

\subsection{Boundary and sampling stability}

Across two additional training seeds, all 26 empirical-boundary runs pass. The one-channel-smaller width fails in 25 of 26 runs; the sole exception is KMNIST at NMSE 0.009986, immediately below the bound. Mean NMSE is 0.00928 at the boundary and 0.01043 one channel below. The \method\ prediction passes in 21 of 26 runs. Its five failures comprise one KMNIST seed and both seeds for PneumoniaMNIST and BreastMNIST; the worst NMSE is 0.01023. These are the same adjacent-channel ambiguities seen in the main run rather than scale-selection errors.

\begin{figure}[t]
    \centering
    \includegraphics[width=\linewidth]{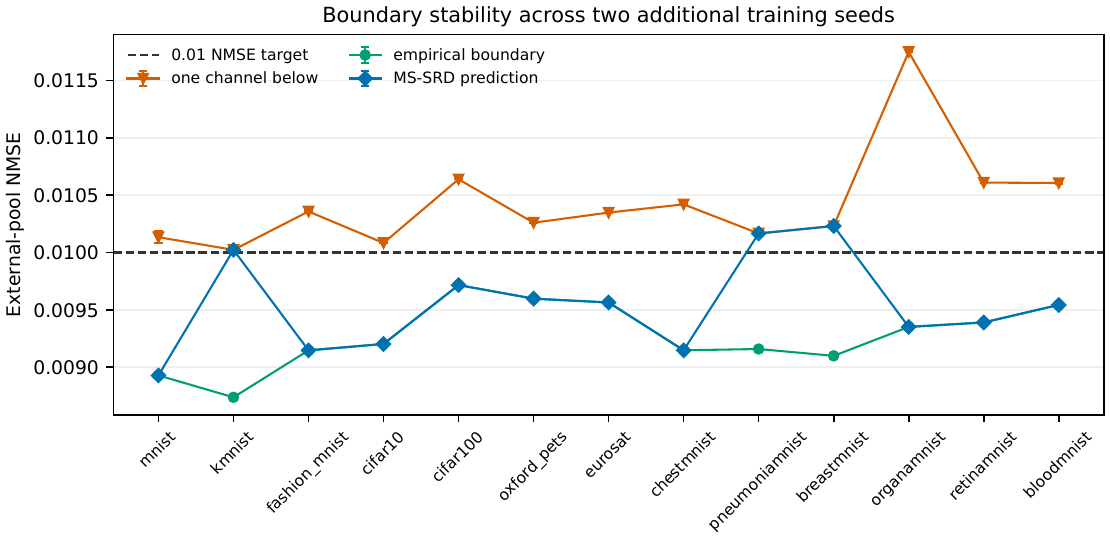}
    \caption{Mean and range across two additional seeds at $\delta=0.01$. Every empirical-boundary run passes; one of 26 one-channel-smaller runs also falls just below the bound.}
    \label{fig:robustness}
\end{figure}

The bootstrap selects the same $(q,c,M)$ in all 20 replicates for seven datasets: KMNIST, Fashion-MNIST, CIFAR-100, ChestMNIST, PneumoniaMNIST, OrganAMNIST, and BloodMNIST. The other six datasets split only between adjacent channel counts at the same scale.

\FloatBarrier
\subsection{The skip-closed U-shaped bottleneck}

Figure~\ref{fig:truecurves} shows the test distortion curves for the architecture in Section~\ref{sec:trueunet}. Every marked retrospective boundary is an adjacent integer crossing: the displayed channel count passes and the next smaller count fails. Independent optimization produces three small non-monotone steps across all sampled curves; the largest increase is 0.00163 NMSE and none changes a reported crossing. For each distortion budget $\delta$, write $c_{\rm MS}(\delta)$ for the \method-predicted channel count and $c_{\rm test}(\delta)$ for this retrospective test boundary.

\begin{table}[b]
\centering
\caption{Relation between \method\ channel counts and protocol-specific test boundaries for the skip-closed U-shaped autoencoder. ``Pass'' counts datasets whose unmodified \method\ width meets the test bound. ``Val. MAE'' compares the validation-selected and retrospective test widths.}
\label{tab:trueunet}
\small
\setlength{\tabcolsep}{5.5pt}
\begin{tabular}{lrrrrrr}
\toprule
$\delta$ & Median $c_{\rm test}/c_{\rm MS}$ & MS MAE & Exact & Within $\pm1$ & Pass & Val. MAE\\
\midrule
0.01 & 1.000 & 1.36 & 5/11 & 9/11 & 7/11 & 0.64\\
0.02 & 0.967 & 1.91 & 5/11 & 7/11 & 11/11 & 0.45\\
0.05 & 0.692 & 4.55 & 1/11 & 2/11 & 10/11 & 0.55\\
0.10 & 0.429 & 4.36 & 2/11 & 2/11 & 11/11 & 0.45\\
\bottomrule
\end{tabular}
\end{table}

At the primary bound $\delta=0.01$, the median empirical-to-predicted ratio is one, the mean absolute width error is 1.36 channels, five predictions are exact, and nine lie within one channel (Table~\ref{tab:trueunet} and Fig.~\ref{fig:relation}). Four predictions fail at the unmodified width, and each corresponding boundary is exactly one channel wider. The two larger errors are conservative overestimates on ChestMNIST and RetinaMNIST.

At $\delta=0.10$, the nonlinear model's median boundary is 42.9\% of the spectral width; the median rises through 0.692 and 0.967 to 1.000 as $\delta$ decreases. Deep nonlinear branches exploit additional structure when appreciable error is allowed, whereas a strict reconstruction bound carries more local directions through the terminal cut. The four budgets therefore trace a systematic transition from nonlinear savings toward agreement with the spectral width.

\begin{figure}[t]
    \centering
    \includegraphics[width=\linewidth]{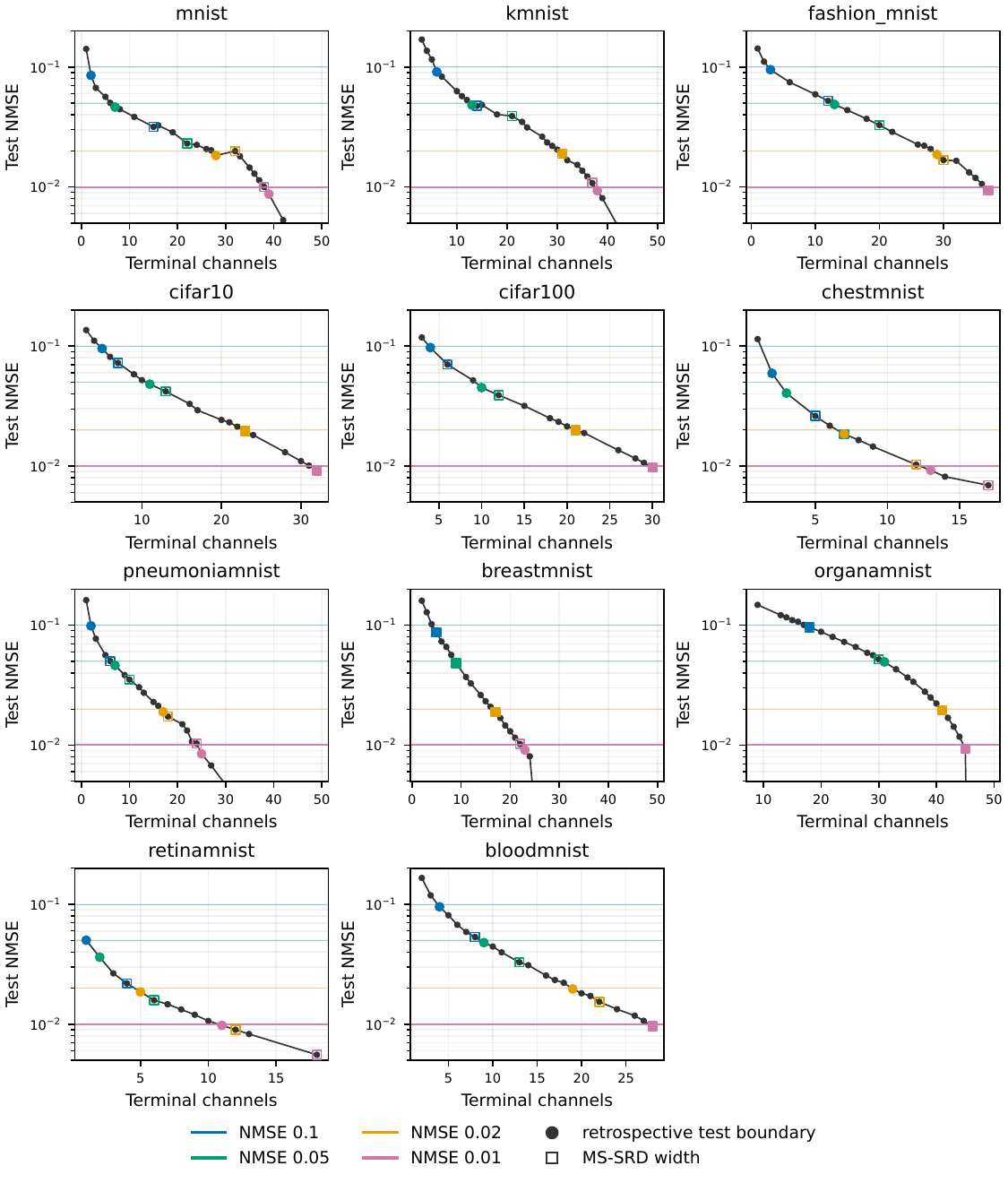}
    \caption{Test NMSE versus terminal channels for the skip-closed U-shaped autoencoder. Filled circles are retrospective adjacent-channel crossings; open squares are \method\ widths. The test boundary is an analysis oracle, not a deployable selector.}
    \label{fig:truecurves}
\end{figure}

\begin{figure}[t]
    \centering
    \includegraphics[width=\linewidth]{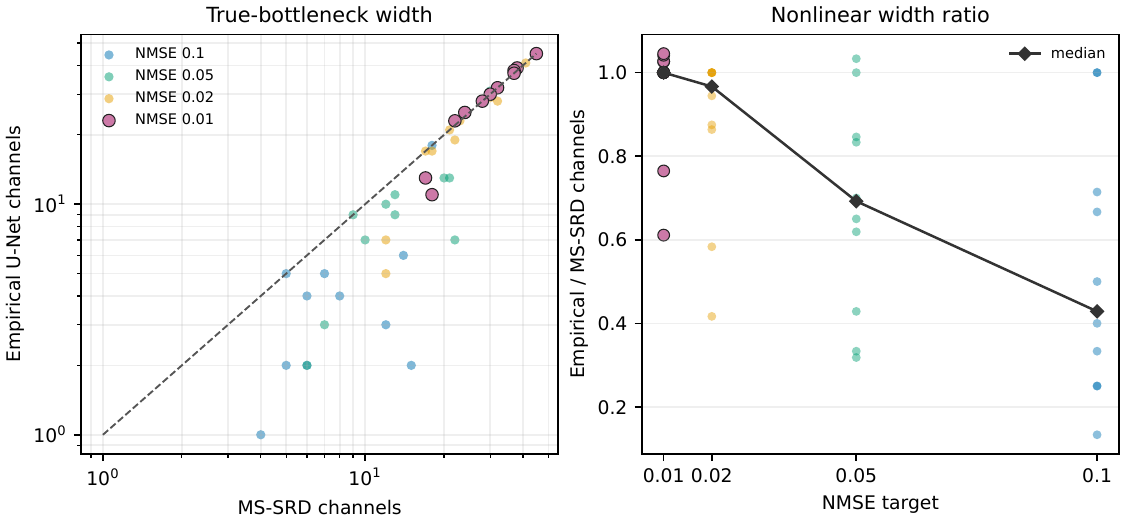}
    \caption{Left: \method\ width versus the protocol-specific test boundary. Right: their ratio by NMSE bound; diamonds connect per-bound medians. The primary $\delta=0.01$ points are emphasized; looser bounds show the change in nonlinear savings.}
    \label{fig:relation}
\end{figure}

For deployment, the architecture first fixes the grid and \method\ supplies $c_{\rm MS}(\delta)$ at that cut. At $\delta=0.01$, every residual underestimation in this study is one channel. A validation search initialized around $c_{\rm MS}$ has mean absolute test-boundary error below 0.7 channels at every evaluated $\delta$, and all 44 selected widths are within three channels of the retrospective boundary. The training-free result therefore narrows the candidate region for the architecture-specific search.

\FloatBarrier
\subsection{A full U-Net bypasses the terminal bottleneck}

Both full-skip controls satisfy $\delta=0.01$ on all eleven datasets, including the models whose terminal latent activation is identically zero. The zero-terminal models have median held-out NMSE 0.000710 (0.0710\%) and worst-case NMSE 0.001089 (0.1089\%) on PneumoniaMNIST, still more than ninefold below the bound. Figure~\ref{fig:unet} compares the zero and one-channel conditions; both remain in the same sub-percent error regime across the eleven datasets.

\begin{table}[b]
\centering
\caption{Full-skip U-Net audit. $M$ is the \method-predicted terminal scalar count at $\delta=0.01$, $S$ is raw skip-path activation count, and zero NMSE is measured with the terminal activation identically zero.}
\label{tab:unet}
\scriptsize
\setlength{\tabcolsep}{4.3pt}
\begin{tabular}{lrrrr@{\hspace{1.5em}}lrrrr}
\toprule
Dataset & $M$ & $S$ & $S/M$ & Zero NMSE & Dataset & $M$ & $S$ & $S/M$ & Zero NMSE\\
\midrule
MNIST & 608 & 21,952 & 36.1 & .0246\% & CIFAR-100 & 480 & 28,672 & 59.7 & .0710\%\\
KMNIST & 592 & 21,952 & 37.1 & .0406\% & ChestMNIST & 272 & 21,952 & 80.7 & .0798\%\\
Fashion-MNIST & 592 & 21,952 & 37.1 & .0457\% & PneumoniaMNIST & 384 & 21,952 & 57.2 & .1089\%\\
CIFAR-10 & 512 & 28,672 & 56.0 & .0654\% & BreastMNIST & 352 & 21,952 & 62.4 & .0866\%\\
OrganAMNIST & 720 & 21,952 & 30.5 & .0465\% & RetinaMNIST & 288 & 21,952 & 76.2 & .0888\%\\
BloodMNIST & 448 & 21,952 & 49.0 & .0808\% & & & & & \\
\bottomrule
\end{tabular}
\end{table}

\begin{figure}[t]
    \centering
    \includegraphics[width=\linewidth]{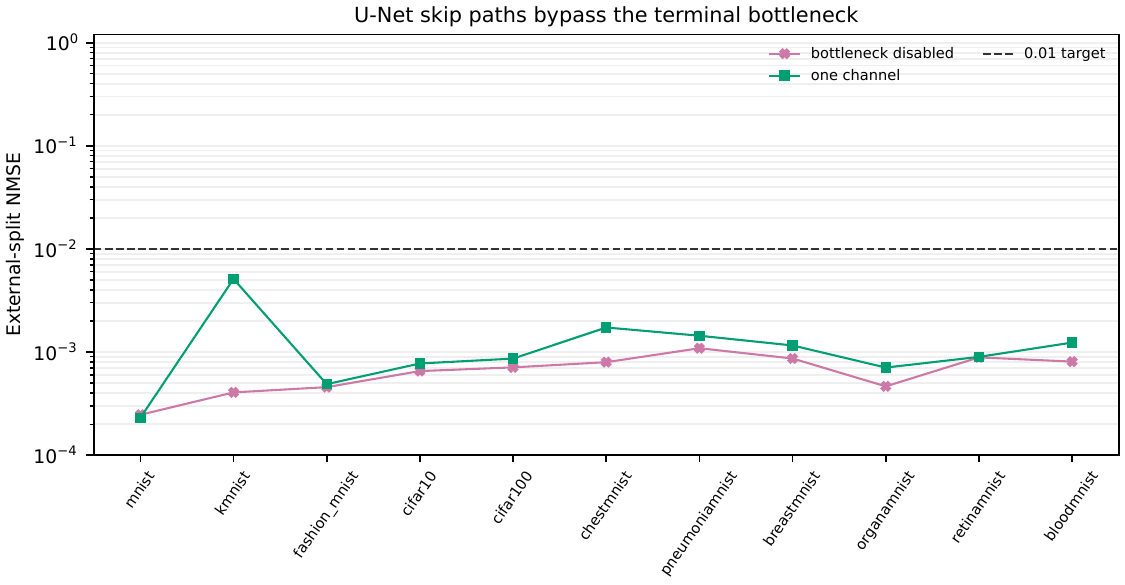}
    \caption{External-split NMSE for full-skip U-Nets with a zeroed terminal path or one terminal channel. Every run lies below $\delta=0.01$.}
    \label{fig:unet}
\end{figure}

The skip audit explains this outcome. The skips transmit between 30.5 and 80.7 times as many raw scalars as the predicted terminal tensor, with a median ratio of 56.0. High-resolution skips provide the decoder with a direct input-dependent route, and the zero-terminal experiment shows that this route alone supports the reported reconstructions.

\clearpage

\section{Discussion}

\subsection{What the result establishes}

For a fixed NMSE bound and block-local model, the covariance eigenvalue tail predicts local reconstruction loss, while spatial repetition converts local rank to tensor size. This gives an exact linear solution and a useful nonlinear initialization.

The U-shaped experiment separates the strict- and loose-bound regimes. At the primary $\delta=0.01$ bound with skips closed, the median empirical-to-predicted channel ratio is one and every underestimation is one channel. The wider-budget experiments show that nonlinear savings grow as the reconstruction constraint is relaxed. In the full-skip control, accurate zero-terminal reconstructions confirm that the skip pathway carries substantial reconstruction capacity.

\method\ uses covariance eigenvalues to count the local directions required by the MSE bound. The associated eigenvectors are needed only for the analytic encoder--decoder map and for initialization of the nonlinear validation model.

\subsection{Limitations}

Three limitations delimit the present result.

\begin{itemize}[label=$\bullet$,leftmargin=1.4em,itemsep=2pt]
    \item \textbf{Model and objective.} The theorem covers shared linear encoders on non-overlapping local blocks under average grayscale pixel MSE. Overlapping or nonlinear operators, color covariance, perceptual and semantic objectives, and worst-case error require corresponding extensions. The reported real-valued scalar count is not a compressed bit rate without a quantizer and entropy model.
    \item \textbf{Architecture dependence.} Candidate spatial cuts are supplied by the architecture: nested-scale dominance makes latent count alone insufficient to choose encoder depth. The nonlinear relation is architecture-specific, and unrestricted U-Net skips bypass the terminal tensor; \method\ does not yet allocate a joint capacity budget across skip pathways and the terminal representation.
    \item \textbf{Statistical and empirical range.} Finite samples perturb eigenvalues near the selection threshold, nonlinear branches can depart from the linear projection, and a shifted test distribution can have a different covariance tail. Bootstrap resampling measures spectral sampling variation but not model-class or train--test effects. The benchmark uses grayscale images up to $64\times64$ with a $96$-pixel geometry stress test, and its external crossings are retrospective comparison oracles rather than unbiased deployment estimates.
\end{itemize}

\pagebreak
\section{Conclusion}

\method\ converts multiscale patch covariance spectra and an NMSE bound into bottleneck channel counts and an activation--parameter Pareto frontier before network training. At $\delta=0.01$, it achieves 0.84\% latent-size MAPE across thirteen nonlinear patch-model experiments; ten predictions are exact and the remaining three differ by one channel. In the skip-closed U-shaped model, the median empirical-to-predicted channel ratio is one, five predictions are exact, and nine are within one channel. Across looser bounds, the experiments quantify how nonlinear savings grow as the reconstruction constraint relaxes. These results establish local covariance spectra as a practical basis for concentrating bottleneck-width search around a small, data-dependent candidate region.

\end{document}